\documentclass[runningheads]{llncs}
\usepackage{graphicx}
\usepackage{amsmath}
\usepackage{amssymb}

\newtheorem{observation}{Observation}

\usepackage{caption} 
\begin{document}
\title{On Conforming and Conflicting  Values}
%\title{On the conformity and conflict of values \thanks{Supported by organization x.}}
\author{Kinzang Chhogyal\inst{1} \and
Abhaya Nayak\inst{1} \and
Aditya Ghose\inst{2} \and
Mehmet Orgun \inst{3}
}
\authorrunning{K. Chhogyal et al.}
\institute{Macquarie University, Sydney, Australia\\
\email{kin.chhogyal@mq.edu.au, abhaya.nayak@mq.edu.au} \and
University of Wollongong, Wollongong, Australia \\
\email{aditya@uow.edu.au }\\ \and
Macquarie University, Sydney, Australia \\
\email{mehmet.orgun@mq.edu.au}
}
\maketitle              % typeset the header of the contribution
\begin{abstract}
Values are things that are important to us. Actions activate values - they either go against our values or they promote our values. Values themselves can either be conforming or conflicting depending on the action that is taken. In this short paper, we argue that values may be classified as one of two types - \emph{conflicting} and \emph{inherently conflicting} values.  They are distinguished by the fact that the latter in some sense can be thought of as being independent of actions. This allows us to do two things: i)  check whether a set of values is consistent and ii) check whether it is in conflict with other sets of values.    

\keywords{Values  \and Conflicting Values \and Value States.}
\end{abstract}
\section{Introduction}

\noindent The pervasiveness of AI in society has benefitted us but has also spurred on much public debate about important issues including ethics and trust in AI systemts. This has resulted in a growing interest in the research on values such as in \emph{value sensitive design} \cite{friedman2013value} where systems are designed by identifying values that are important and then translating them into design requirements \cite{van2013translating}. Others have focussed on the use of values in argumentation \cite{bench2009abstract}, on their relation with norms \cite{Serramia:2018:MVN:3237383.3237891,ghose2012norms}, in characterising opportunistic propensity \cite{10.1007/978-3-319-71682-4_13}, and in plan selection for BDI agents \cite{ijcai2017-26}. However, in most of these works except  \cite{ijcai2017-26}, the treatment of values is limited - they are cast as abstract entities, usually with some preference ordering - and the research is more on the use of values as a means to an end rather than on the values themselves. It seems then that we must try to get a deeper understanding of values and in this paper, we look at one particular aspect of values.\\

\noindent In Psychology, Schwartz's \emph{Theory of Basic Human values} \cite{schwartz2012overview} paints a richer picture of values. It is assumed that all values share certain features. Of them, three are particularly relevant: i) values can be activated and cause emotions to arise, ii) they can influence our goals and therefore the choice of our actions and, iii) they may have a preference ordering. He also identifies ten broad (abstract) values under which it is assumed most concrete values fall. What is of most interest to us is the dynamics of values; it is stated that  \emph{actions in pursuit of any value have consequences that conflict with some values but are congruent with others} \cite{schwartz2012overview}. This paper is motivated by that particular statement. We begin by presenting a simple formalisation that captures what it means for values to conform or conflict with each other when actions are executed. We show that adopting this formalisation, leads us to special pairs of values that are always in conflict which we call \emph{inherently conflicting} values and we end by briefly discussing some implications of this work.

\section{Values}

\noindent We assume there is a set of all values, $\mathcal V = \{a, b, \hdots\}$, from which agents draw their values. These values represent concrete values which can be thought to fall under the ten broad values \cite{schwartz2012overview}. We also assume there is a set of possible actions $\mathcal A =\{a, a', \hdots \}$ that agents can execute. Let $\mathcal{S}$ be the set of states that the world can be in and by $S(a)$ we denote a subset of $\mathcal{S}$ where $a$ is executable.\footnote{The symbol $a$ can represent both a value or an action but it is usually clear from the context what $a$ is referring to.}

\subsection{Conformance and Conflict}

We begin with the definition of a value state which is inspired by the one in \cite{ijcai2017-26}.
\begin{definition}
Give a value $v$, the value state of $v$ is denoted as $VS(v)$ where $VS(v) \in \mathbb{N}^+$.\footnote{The actual representation of the value state is not important and neither are the bounds. What is important is that value states can increase and decrease.}
\end{definition}

\noindent We said in the Introduction that one of the properties of values is that they can be activated and stir up emotions \cite{schwartz2012overview}. Actions have the potential to activate values which results in an increase or decrease in their value state. We say potential because the state under which the action is executed determines which and how many values are activated. In some states under certain actions, all values might get activated and in some few or none of the values might get activated. An action that causes the value state of a value to increase is interpreted as one that \emph{promotes} the value where as if it decrease the value state, it \emph{acts against} the value.  We use the following notation to show how the value state of a value $v$ changes given an action $a$ and a state $s \in S(a)$:

\begin{equation}
\begin{aligned}
(s, a, v) \rightarrow v^{*}
\end{aligned}
\end{equation}
where $*$ is $\uparrow$, $\downarrow$ or  $\leftrightarrow$ and indicates whether the value state of $v$ has increased, decreased or remained unchanged respectively. \\

\begin{definition}[Indifference]
A value $v$ is indifferent to an action $a$  in state $s \in S(a)$ if whenever $a$ is executed in $s$, $VS(v)$ remains unchanged.
\begin{equation}
\begin{aligned}
(s, a, v) \rightarrow v^{\leftrightarrow}
\end{aligned}
\end{equation}
\end{definition}

\noindent We take it that it is not possible to have a value $v$ that is indifferent to every action in every state. This would make $v$  meaningless and it shouldn't have been included as a value in the first place. We state it as the following condition:

\hspace{0mm}\\
\noindent \textbf{Condition-1}: For any value $v$, there must be at least one action $a$ and one state $s \in S(a)$ which activates $v$.
\begin{definition}[Conflicting Values]
Given an action $a$, a state $s \in S(a)$, and two values $v, v' \in \mathcal{V}$, if whenever $a$ is executed in $s$, $VS(v)$ increases (decreases) and $VS(v')$ decreases (increases) then we say $v$ and $v'$ are conflicting values with respect to $a$ and $s$.
\begin{equation}
\begin{aligned}
 & (s, a, v) \rightarrow v^{\uparrow} \text{ and } (s, a, v') \rightarrow v'^{\downarrow}  & or   \\
 & (s, a, v) \rightarrow v^{\downarrow} \text{ and } (s, a, v') \rightarrow v'^{\uparrow} &
\end{aligned}
\end{equation}
\end{definition}
\begin{example}
Consider that you value both \emph{frugality} and \emph{quality}. If you decide to buy a flimsy plastic table that costs \$100 over a sturdy wooden table that costs \$200, it increases the value state of frugality but decreases the value state of quality. Thus, they are conflicting values in this situation.
\end{example}
\begin{definition}[Conforming Values]
Given an action $a$, a state $s \in S(a)$, and two values $v$ and $v'$, if whenever $a$ is executed in $s$, $VS(v)$ increases (decreases) and $VS(v')$ also increases (decreases) then we say $v$ and $v'$ are conforming values with respect to $a$ and $s$.
\begin{equation}
\begin{aligned}
 & (s, a, v) \rightarrow v^{\uparrow} \text{ and } (s, a, v') \rightarrow v'^{\uparrow}  & or   \\
 & (s, a, v) \rightarrow v^{\downarrow} \text{ and } (s, a, v') \rightarrow v'^{\downarrow} &
\end{aligned}
\end{equation}
\end{definition}
\begin{example}
Again take the values of frugality and quality but this time you notice that the wooden table is being offered at a discount of 50\%. If you buy the wooden table, it will increase the value state of both frugality and quality. The two values are conforming in this situation.
\end{example}
\begin{definition}[Inherently Conflicting Values]
Given two values $v, v' \in \mathcal{V}$, if for all actions $a \in \mathcal{A}$ and for all states $s \in S(a)$, $v$ and $v'$ are conflicting values or if both $v$ and $v'$ are indifferent, we say $v$ and $v'$ are inherently conflicting values. We also say $v$ inherently conflicts with $v'$ and vice versa.
\end{definition}
\begin{example}
Consider that you value a \emph{free market economy} over a \emph{regulated economy}. If you are a legislator and you decide to support legislation that increases the tariff on imported goods, it decrease the value state of free market economy and increases the value state of regulated economy. In fact, any action that promotes one will go against the other and thus they are always in conflict.
\end{example}

\noindent Note that because of Condition-1 previously stated, it is not possible to have inherently conflicting values that are only indifferent and not conflicting. By conflicting we  mean values that are conflicting but not inherently conflicting unless otherwise stated.

\begin{definition}[Inherently Conforming Values]
Given two values $v$ and $v'$, if for all actions $a \in \mathcal{A}$ and for all states $s \in S(a)$, $v$ and $v'$ are conforming values or if both $v$ and $v'$ are indifferent, we say $v$ and $v'$ are inherently conforming values. We say $v$ inherently conforms with $v'$ and vice versa.
\end{definition}
%

%\noindent Note that according to the definitions above, it is possible that two values are inherently conforming and do not conflict, therefore inherently conflict,  with any other value. 

%
\begin{observation}
Given two values $v'$ and $v''$  that are are inherently conforming and another value $v$:
\begin{itemize}
\item[a)] if $v' (v'')$  conflicts (conforms) with $v$, then  $v''(v')$ also conflicts(conforms) with $v$, and
\item[b)] if $v' (v'')$ inherently conflicts with $v$, then  $v''(v')$ also inherently conflicts with $v$.
\label{obs:twoInherentlyConformingValuesSame}
\end{itemize}
\end{observation}
\begin{observation}
Given two values $v'$ and $v''$  that are are inherently conflicting and another value $v$:
\begin{itemize}
\item[a)] if $v' (v'')$  conflicts (conforms) with $v$, then  $v''(v')$  conforms(conflicts) with $v$, and
\item[b)] if $v' (v'')$ inherently conforms with $v$, then  $v''(v')$ also inherently conflicts with $v$.
\end{itemize}
\end{observation}
\begin{proposition}
If two values $v'$ and $v''$ are inherently conflicting with $v$, then $v'$ and $v''$ are inherently conforming.
\label{prop:inherentlyConflictingIsInherentlyConforming}
\end{proposition}
\begin{proof}
Assume the antecedent. For contradiction, assume $v'$ and $v''$ are not inherently conforming. So there must be an action $a$ and a state $s \in S(a)$ such that $(s, a, v') \rightarrow v'^\uparrow$ and $(s, a, v'') \rightarrow v''^\downarrow$ or $(s, a, v') \rightarrow v'^\downarrow$ and $(s, a, v'') \rightarrow v''^\uparrow$. However, since $v$ is inherently conflicting with $v'$ and $v''$, it must be that either $(s, a, v) \rightarrow v^\uparrow$ or $(s, a, v) \rightarrow v^\downarrow$ and both $v'$ and $v''$ must simultaneously either increase or decrease their value state, which results in a contradiction. $\Box$
\end{proof}

\noindent This could be contentious but it seems that representing two inherently conforming values separately doesn't offer us much; from Observation \ref{obs:twoInherentlyConformingValuesSame} we can see that whatever we can say of one value - with respect to the other values they conflict or confirm with - is true of the other.  This suggests that perhaps two inherently conforming values should be collapsed into one as they virtually represent the same value.  On the other hand,  if we have multiple values that are inherently in conflict with another value, then from Proposition \ref{prop:inherentlyConflictingIsInherentlyConforming}, we know they are inherently conforming, and again from Observation \ref{obs:twoInherentlyConformingValuesSame} we get the same argument for collapsing them into one. On the basis of this discussion, we introduce some further conditions that we assume to hold henceforth.

\hspace{0mm}\\
\noindent \textbf{Condition-2}: For any value $v$, if it inherently conforms with a value $v'$, then $v$ and $v'$ are the same. \\ \\
\noindent \textbf{Condition-3}: For any value $v$, if it inherently conflicts with $v'$, then $v'$ is the only one it inherently conflicts with.
\begin{example}
Let $\mathcal S =\{s1, s2\}$, $\mathcal A = \{a1, a2\}$ and let $V = \{ a, b, c, d \}$. Let $a1$ and $a2$ be executable in every state, i.e. $S(a1) = S(a2) = \mathcal{S}$. The dynamics of the values are shown in Table \ref{table-ex-1}. Values that are indifferent to actions are not shown in the state-action cells. We state some relations: 
\begin{enumerate}
\item $a$ and $b$ are inherently conflicting because in every state where $a^\uparrow$, $b^\downarrow$ or $a^\downarrow$, $b^\uparrow$ or $a^\leftrightarrow, b^\leftrightarrow$.
\item $a$ and $c$ are conflicting in $(s1, a')$ as $a^\uparrow, c^\downarrow$ and conforming in $(s1, a'')$ as $a^\downarrow, c^\downarrow$.
\item $a$ and $d$ are not inherently conflicting because in $(s1, a'')$ and $(s2, a'')$, we see $d^\uparrow$ and $d^\downarrow$ respectively, whereas $a^\leftrightarrow$ in both cases.
\end{enumerate}
\label{ex-1}
\end{example}
\vspace{-5mm}
\begin{table}[h]
\centering
\setlength{\tabcolsep}{5pt}
\renewcommand{\arraystretch}{1.5}
\begin{tabular}{| c| p{3cm} | p{3cm} |}
\hline
      & \textbf{state} $s1$                                                & \textbf{state} $s2$                                           \\ \hline
\textbf{action} $a'$  & $a \uparrow, b\downarrow, c\downarrow, d\downarrow$ & $a\downarrow, b\uparrow, c\downarrow, d\uparrow$ \\ \hline
\textbf{action} $a''$ & $c\downarrow, d\uparrow$                            & $d\downarrow$                                  \\ \hline
\end{tabular}
\caption{Conforming, Conflicting and Inherently Conflicting Values. Details in Ex. \ref{ex-1}.}
\label{table-ex-1}
\vspace{-5mm}
\end{table}

\noindent Given a set  $\mathcal A$ of possible actions and the set $\mathcal{S}$ of all states, we can divide the set of all values $\mathcal{V}$, into two sets: one set containing all pairs of inherently conflicting values and the second set consisting of the remaining values. The symbol  $\mathcal V^\perp \subseteq \mathcal{V}$ represents the set of all inherently conflicting values. From Condition-3, we know that for each value in $v \in \mathcal{V}^\perp$, $v$ has exactly one inherently conflciting $v' \in \mathcal{V}$ which we will denote as $\overline{v}$. Pairs of inherently conflicting values are special because even though we know they are activated by actions, since they are always in conflict or indifferent to any action, this allows us to talk about them without mentioning actions and in this sense they can be seen as being independent of actions. This perspective allow us to do two things: 
\begin{itemize}
\item[a)] given a set of values, it allows us to define what it means for that set to be consistent and,
\item[b)] given two sets of values, it allows us to define when one value set is in conflict with another. 
\end{itemize}
This would not be possible if we tried doing the same thing with conflicting values without also talking about the particular actions involved. 

\begin{definition}
A set of values $V$ is inconsistent iff there exists values $v, v' \in V$ such that $v, \ v' \in \mathcal{V}^\perp$ and $v' =  \overline{v}$. Otherwise, it is consistent.
\end{definition}
\begin{definition}
Two sets of values $V$ and $V'$ are conflicting iff there exists values $v \in V$ and $v' \in V'$ such that $v, v' \in \mathcal{V}^\perp$ and $v' =  \overline{v}$. Otherwise, $V$ and $V'$ are non-conflicting.
\end{definition}
\begin{example}
Let $\mathcal{V} = \{a, \overline{a}, b, c, \overline{c}, d,  \overline{d}, e, f\}$. We have $\mathcal{V}^\perp = \{a, \overline{a}, c, \overline{c}, d, \overline{d}\}$ . Also, let $V = \{a, b, c\}$, $V'=\{\overline{a}, d, e \}$ and $V'' = \{ d, \overline{d}, f\}$. We can say the following:
\begin{enumerate}
\item $V$ and $V'$ are consistent as neither contain inherently conflicting values. 
\item $V''$ is inconsistent as it contains both $d$ and $\overline{d}$.
\item $V$ and $V'$ are conflicting because of $a \in V$ and $\overline{a} \in V'$.
\item $V$ and $V''$ are non-conflicting.
\end{enumerate}
\end{example}

\section{Discussion and Conclusion}
We conclude this paper by briefly addressing two issues that result from our presentation:
\begin{enumerate}
\item When talking about the dynamics of values, aside from the examples that were provided, the definition of conforming, conflicting and inherently conflicting values was entirely based on actions and states with little mention of agents. This alludes to the idea that values could exist as separate entities outside the agent and yet we know without agents, values would be meaningless. We don't really see a problem to this separation and it has been done previously in other areas of AI. For instance, consider the various \emph{action languages} \cite{gel98} used for talking about the effects of actions. Even though actions are executed by agents, actions are generally talked about in terms of their pre-conditions and post-conditions without referring to any agent at all.
\item We said that there are inherently conflicting values and this bring up question of why don't all values have one that they inherently conflict with? If we start with a set of values $\mathcal{V}$, and there is a value $v$ that doesn't have an inherently conflicting value, there is nothing to stop us from introducing a new value  $\overline{v}$ in $\mathcal{V}$? The answer to us it seems is that pairs of inherently conflicting values correspond to values that we find in society that occur ``naturally'' and are in conflict. For example, take \emph{pro-choice} and \emph{pro-life}, values related to abortion - any action that promotes one is clearly going to go against the other. On the other hand, it is hard to think of a value like \emph{healthy lifestyle} as having an inherently conflicting value. If we try to construct one, we end up with what feels like an ``unnatural'' and ``artificial'' value like \emph{unhealthy lifestyle} which no person would hold.
\end{enumerate}

In this short paper, we presented a simple formalisation of the dynamics of values and argued that there might be certain pairs of values that are inherently in conflict. We hope this paper will stimulate further discussion on the meaning of conflicts between values and on the topic in general.

%\begin{definition}
%Given a set of values $V$, an action $a$ and state $s \in S(a)$:
%
%\begin{equation*}
%\begin{aligned}
%V^{\uparrow} & = \{v\} \ | \ (s, a, v) \rightarrow v^{\uparrow}, \\
%V^{\downarrow} & =   \{v\} \ | \ (s, a, v) \rightarrow v^{\downarrow}, \text{ and }  \\
%V^{\leftrightarrow} & =   \{v\} \ | \ (s, a, v) \rightarrow v^{\leftrightarrow}. \\
%\end{aligned}
%\end{equation*}
%where $v \in V$.
%\end{definition}
%

%\noindent The following properties follows:
%
%\begin{enumerate}
%\item $V = V^{\uparrow} \cup V^{\downarrow} \cup V^{\leftrightarrow} $
%\end{enumerate}
%

%\cite{ref_article1,ref_lncs1,ref_book1},
%\cite{ref_article1,ref_book1,ref_proc1,ref_url1}.

\newpage
 \bibliographystyle{splncs04}
 \bibliography{mybibliography}

\end{document}